\def\year{2021}\relax
\documentclass[letterpaper]{article} 
\usepackage{aaai21}  
\usepackage{times}  
\usepackage{helvet} 
\usepackage{courier}  
\usepackage[hyphens]{url}  
\usepackage{graphicx} 
\usepackage{graphicx}  
\usepackage{natbib}  
\usepackage{caption} 
\usepackage{amsmath,amssymb}
\newtheorem{theorem}{Theorem}
\newtheorem{lemma}{Lemma}
\newtheorem{claim}{Claim}

\newtheorem{definition}{Definition}

\renewcommand{\phi}{\varphi}
\renewcommand{\epsilon}{\varepsilon}
\renewcommand{\[}{\llbracket}
\renewcommand{\]}{\rrbracket}
\mathchardef\mhyphen="2D 

\newenvironment{proof-of-claim}{\noindent{\em Proof of Claim.}}{\hfill $\boxtimes\hspace{2mm}$\linebreak}

\newcommand{\B}{{\sf B}}

\renewcommand{\phi}{\varphi}

\usepackage{booktabs}
\newenvironment{proof}{\noindent{\sc Proof.}}{\hfill $\boxtimes\hspace{2mm}$\linebreak}
\newcommand{\qed}{\hfill $\boxtimes\hspace{1mm}$}

\usepackage{stmaryrd}
\renewcommand{\[}{\llbracket}
\renewcommand{\]}{\rrbracket}

\title{Trolley Problem Notes}
\title{Logic of Ethical Dilemmas}
\title{Ethical Dilemmas and Limits of Morality}
\title{Trolley Logic: Ethical Dilemmas\\ and Limits of Morality}
\title{The Trolley Modality: A Logic of Ethical Dilemmas}
\title{The Trolley Modality}
\title{Coalition Power Ethical Dilemmas}

\title{Strategic Coalitions and Ethical Dilemmas}

\title{Ethical Dilemmas of Strategic Coalitions}

\title{Ethical Dilemmas in Strategic Games}

\author {
    Pavel Naumov,\textsuperscript{\rm 1}
    Rui-Jie Yew\textsuperscript{\rm 2}\\
}
\affiliations {
    \textsuperscript{\rm 1} King's College, Pennsylvania, USA \\
    \textsuperscript{\rm 2} Scripps College, California, USA \\
    pgn2@cornell.edu, ryew8098@scrippscollege.edu
}

\begin{document}

\maketitle
\begin{abstract}
An agent, or a coalition of agents, faces an ethical dilemma between several statements if she is forced to make a conscious choice between which of these statements will be true. This paper proposes to capture ethical dilemmas as a modality in strategic game settings with and without limit on sacrifice and for perfect and imperfect information games. The authors show that the dilemma modality cannot be defined through the earlier proposed blameworthiness modality. The main technical result is a sound and complete axiomatization of the properties of this modality with sacrifice in games with perfect information.
\end{abstract}



\section{Introduction}

In this paper we study ethical dilemmas faced by agents and coalitions of agents in multiagent systems. As an example, consider the two diagrams in Figure~\ref{alice bob dilemma figure}. In the situation depicted in the left diagram, an agent must choose between action left ($L$) and action right ($R$). These actions will result in the death of Alice and Bob, respectively. The right diagram adds an additional neutral action ($N$) that results in the system nondeterministically transitioning either in state $u$ or state $v$ and killing Alice or Bob, respectively.

\begin{figure}[ht]
\begin{center}
\scalebox{0.4}{\includegraphics{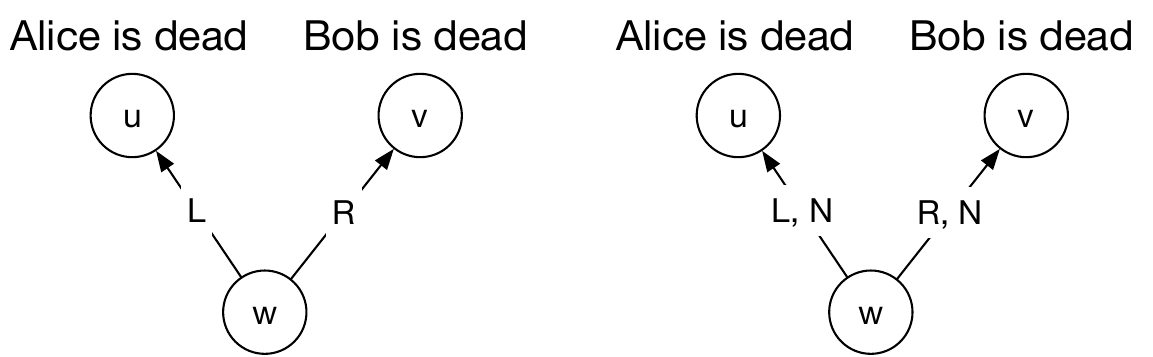}}
\caption{Two situations.}\label{alice bob dilemma figure}
\end{center}
\end{figure}

The situations represented by these two diagrams are similar in many respects. In both of them, in state $w$ the agent has a {\em strategy} to kill Alice (action $L$) and a strategy to kill Bob (action $R$). Additionally, in both settings, the agent will be {\em blamed} for the same outcomes. To claim this, we use an oft-cited~\cite{w17} definition of blameworthiness through the principle of alternative possibilities: ``a person is morally responsible for what he has done only if he could have done otherwise"~\cite{f69tjop}. For example, if the system transitions from state $w$ to state $u$ on either of the diagrams, then the agent will be blamed for the death of Alice because the agent had a strategy (action $R$) to prevent Alice's death. However, the agent is not blamable for the statement ``either Alice or Bob is dead'', because, in both diagrams, the agent does not have a strategy to prevent the statement from being true.

However, there is a difference in the types of choices the agent must make in these two diagrams. In the left diagram, the agent has to make a hard  choice between either consciously killing Alice or consciously killing Bob. On the right diagram, the agent can avoid this choice by selecting action $N$. We say that, on the left diagram, the agent is facing a moral dilemma between killing Alice and killing Bob, while on the second diagram the agent does not. 

\begin{figure}[ht]
\begin{center}
\scalebox{0.7}{\includegraphics{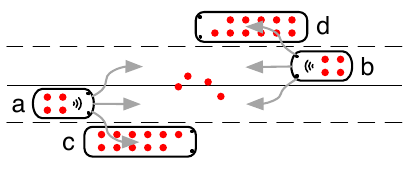}}
\caption{Road traffic situation.}\label{intro figure}
\end{center}
\end{figure}
As another example, consider the traffic situation depicted in Figure~\ref{intro figure}. Here, four pedestrians (red circles in the middle) are stranded on a busy four-lane highway. Self-driving cars $a$ and $b$ are on the path to run them over. It is too late for either of the cars to stop. Car $a$ has three options: to pull left, to keep driving straight, or to pull right into bus $c$. Similarly, car $b$ can pull left, drive straight, or pull right into bus $d$. 

\begin{table}[ht]
\begin{center}
\begin{tabular}{ c|ccc } 
 \toprule
 $a\backslash b$	 & left & straight & right \\ \midrule
 left & $\phi_p$ & $\phi_a\wedge\phi_b$ & $\phi_b\wedge\phi_d$ \\ 
 straight & $\phi_a\wedge\phi_b$ & $\phi_p$ & $\phi_b\wedge\phi_d$ \\ 
 right & $\phi_a\wedge\phi_c$ & $\phi_a\wedge\phi_c$ &  $\phi_a\wedge\phi_b\wedge\phi_c\wedge\phi_d$\\ 
\bottomrule
\end{tabular}
\end{center}
\caption{Strategic game between cars $a$ and $b$.}\label{intro table}
\end{table}

Table~\ref{intro table} shows different outcomes of this strategic game between players $a$ and $b$. In this table, letters $\phi_p$, $\phi_a$, $\phi_b$, $\phi_c$, and $\phi_d$ represent 
the death of the pedestrians, the passengers in car $a$, the passengers in car $b$, some of the passengers in bus $c$,
and
some of the passengers in bus $d$,
respectively.

In this situation car $a$ faces a choice: it can either pull right or not pull right. In the former case, it is {\em guaranteed} to kill its own passengers  as well as some of the passengers in bus $c$. In the latter case one of the following is {\em guaranteed} to happen: either pedestrians will die, cars $a$ and car $b$ will collide, or car $b$ and bus $d$ will collide. In other words, car $a$ is facing a dilemma between an action that will force $\phi_a\wedge\phi_c$ and the action that will force $\phi_p\vee(\phi_a\wedge\phi_b)\vee(\phi_b\wedge\phi_d)$. We denote this dilemma of car $a$ by formula
$$
[a:\phi_a\wedge\phi_c,\phi_p\vee(\phi_a\wedge\phi_b)\vee(\phi_b\wedge\phi_d)].
$$
Similarly, car $b$ is facing dilemma
\begin{equation}\label{b's dilemma}
    [b:\phi_p\vee(\phi_a\wedge\phi_b)\vee(\phi_a\wedge\phi_c),\phi_b\wedge\phi_d].
\end{equation}

Philosophers distinguish several approaches to morality. Consequential ethicists judge the moral acceptability of actions based on their outcomes. For example, a utilitarian (consequential) ethicist might say that it is morally unacceptable to kill more than a certain number of civil casualties in a military operation. On the other hand, absolute ethicists find certain actions morally unacceptable no matter what their results are. For example, a Kantian ethicist might object to pushing the lever in a trolley dilemma in order to sacrifice one person and save five. Many of such moral constraints can be modeled using the {\em cost of sacrifice} approach that we propose in this paper. We assign a cost of sacrifice to each action and specify the limit on the acceptable sacrifice for each agent as a subscript of the dilemma modality. For a utilitarian facing an ethical dilemma, the sacrifice is the number of civil casualties. For the absolute ethicist, sacrifice is $+\infty$ for all actions that are not morally acceptable.

The same approach can be used to model constraints imposed by laws, regulations, or company policies. For example, recently introduced German regulations for autonomous vehicles state that, when confronted with the choice between the death of a human being and damage to property, a self-driving car must always choose the latter~\cite{ec17fmtdi}. 
In this case,
cost of an action can be defined as the minimal number of people the action is guaranteed to kill above the unavoidable minimum. For example, if a hypothetical car is choosing between four actions that are guaranteed to kill $5$, $9$, $5$, and $7$ people respectively, then the costs of these actions are $0$, $4$, $0$, and $2$. The German rule would require a car to select one of the two actions with zero cost. 

According to Car and Driver magazine, Mercedes-Benz manager of driver assistance systems and active safety Christoph von Hugo stated that  
``If you know you can save at least one person, at least save that one. Save the one in the car. ... If all you know for sure is that one death can be prevented, then that’s your first priority.''~\cite{t16candd}. This potential policy for future Mercedes-Benz self-driving vehicles defines the cost of an action as the minimal number of people {\em inside the vehicle} the action is guaranteed to kill above the unavoidable minimum. The policy also sets the allowed sacrifice in terms of this cost to zero. 

Let us now assume that car $a$ (but not car $b$) in Figure~\ref{intro figure} is a self-driving vehicle made by Mercedes-Benz. Under the above policy\footnote{Mercedes-Benz later retracted this policy stating that ``to make a decision in favor of one person and thus against another is not legally permissible in Germany''~\cite{o16jalopnik}.}, the car will never choose to pull into bus $c$. Thus, car $a$ is now facing a vacuous one-option dilemma that any action that the car takes will result in statement $\phi_p\vee(\phi_a\wedge\phi_b)\vee(\phi_b\wedge\phi_d)$ being true. We write this as
$$
[a:\phi_p\vee(\phi_a\wedge\phi_b)\vee(\phi_b\wedge\phi_d)]_{a,b\mapsto 0,+\infty},
$$
where sacrifice function $a,b\mapsto 0,+\infty$ assigns the maximal sacrifice that each agent is ready to tolerate. In our case, the limit on the number of people inside the vehicle that car $a$ is ready to sacrifice is $0$. Car $b$ does not have any fixed sacrifice limit, which we interpret as the value of the sacrifice function for agent $b$ being $+\infty$. Note that although agent $b$ in this situation does not have a sacrifice limit, the limit on the sacrifice of agent $a$ modifies not only $a$'s dilemma but $b$'s as well. Compare the following statement to statement~(\ref{b's dilemma}):
$$
[b:\phi_p\vee(\phi_a\wedge\phi_b),\phi_b\wedge\phi_d]_{a,b\mapsto 0,+\infty}.
$$

If self-driving cars $a$ and $b$ decide to cooperate and make a joint decision, then instead of two individual dilemmas they face a single {\em multiagent} ethical dilemma. Let us first assume that neither of these two vehicles is a Mercedes-Benz. Thus, they can either (i) kill all pedestrians by driving in two different lanes, (ii) kill passengers in cars $a$ and $b$ by sending both vehicles for a head-on collision,  (iii) collide car $a$ with bus $c$, or (iv) collide car $b$ with bus $d$:   
$$
[a,b:\phi_p,\phi_a\wedge\phi_b,\phi_a\wedge\phi_c,\phi_b\wedge\phi_d]_{a,b\mapsto 
+\infty,+\infty}.
$$

Recall that if $a$ is a Mercedes-Benz car, then it is restricted from pulling right into bus $c$ because this action is guaranteed to kill passengers inside car $a$. Note, however, that, though there is always the chance that car $b$ pulls left and crashes into car $a$, there is no guarantee that car $a$ will collide with car $b$. Thus, the same Mercedes-Benz policy does not restrict car $a$ from pulling left. Let us now consider the case where both $a$ and $b$ are Mercedes-Benz vehicles making a joint decision. Does the policy restrict them from a joint decision under which car $a$ drives straight and car $b$ pulls left? In other words, is the policy a restriction on individual actions of Mercedes-Benz cars or a restriction on joint decisions of all Mercedes-Benz vehicles? If the former is true, as it is in the formal model described in this paper, then coalition $\{a,b\}$ is facing a dilemma between killing all pedestrians and a head-on collision:
$
[a,b:\phi_p,\phi_a\wedge\phi_b]_{a,b\mapsto 0,0}.
$
If the latter is true, then the two vehicles must either drive straight or both of them must pull left. In any case, the pedestrians will die:
$
[a,b:\phi_p]_{a,b\mapsto 0,0}.
$
Although Christoph von Hugo did not explicitly specify that this policy applies to individual vehicles, we think this is the case. If the policy were to apply to coalitions, then one might face a new version of the trolley dilemma when a fleet of Mercedes-Benz vehicles might choose to sacrifice the life of a passenger in a Mercedes-Benz vehicle in order to safe the lives of two passengers in another Mercedes-Benz vehicle. This seems to contradict von Hugo's claim that the first priority should be the prevention of even one death of a passenger in a Mercedes-Benz self-driving vehicle. 



\section{Overview}

The rest of this paper is organized as follows. First, we describe the syntax and formal semantics of the ethical dilemma modality $[C\!:\!\phi_1,\dots,\phi_n]_s$ in a strategic game setting. Then, we review literature on ethical dilemmas and compare the dilemma modality to the earlier studied blameworthiness and coalition power modalities. In particular, we show that the dilemma modality cannot be defined through the blameworthiness modality even in the single-agent setting without sacrifice. We also demonstrate how our definition of ethical dilemma can be extended to games with imperfect information. Finally, we give a complete axiomatization of our modality in the perfect information case. The proof of completeness is in the full version of this paper~\cite{ny19arxiv}.



\section{Strategic Game with Normalized Costs}

Recall from the introduction that if an autonomous vehicle is confronted with the choice between four actions that are guaranteed to kill $5$, $9$, $5$, and $7$ people respectively, then the costs of these actions are $0$, $4$, $0$, and $2$. In other words, we assume that costs are ``normalized'' so that at least one of them is zero.


\begin{definition}\label{normalized definition}
A function $f:X\to[0,+\infty]$ is normalized if there is an element $x\in X$ such that $f(x)=0$.
\end{definition}

The strategic games with normalized costs that we define bellow are very similar to ``resource-bounded action frames'' used in the semantics of Resource Bounded Coalition Logic~\cite{alnr11jlc}. By $X^Y$ we denote the set of all functions from set $Y$ to set $X$. Throughout the paper we assume a fixed set of propositional variables and a fixed set of agents $\mathcal{A}$.

\begin{definition}\label{game}
A game is a tuple $(W,M,\Delta,|\cdot|,\pi)$, where
\begin{enumerate}
    \item $W$ is a set of states,
    \item $\Delta$ is a set of ``actions'',
    \item $M\subseteq W\times \Delta^\mathcal{A}\times W$ is a relation, called ``mechanism'',
    \item $|d|_w^a\in [0,+\infty]$ is the ``cost'' of action $d\in \Delta$ for $a\in\mathcal{A}$ in state $w\in W$, such that $|d|_w^a$ is {\bf \em normalized} as a function of action $d$ for any fixed values $a\in\mathcal{A}$ and $w\in W$,
    \item $\pi(p)$ is a subset of $W$ for each propositional variable $p$.
\end{enumerate}
\end{definition}

We refer to functions in set $\Delta^\mathcal{A}$ as complete action profiles of the game. Informally, mechanism $M$ captures the rules of the game. Namely, $(w,\delta,u)\in M$ if under complete action profile $\delta$ the game can transition from state $w$ to state $u$. Our semantics is slightly more general than in ~\cite{alnr11jlc} because we assume that mechanism is a relation and not necessarily a function. In other words, we allow a complete action profile to transition the game into one of several different states. 
Our approach also allows some complete action profiles to result in no next state at all. We interpret this as a termination of the game. We normalize the costs of actions in order to avoid a situation when, for a given sacrifice, an agent would not have any actions to choose from.
Note that Definition~\ref{game} allows actions with infinite costs. We further discuss such actions in the conclusion.

\section{Syntax}\label{Syntax section}

In this paper we assume a fixed set $\mathcal{A}$ of agents. By a coalition we mean any nonempty subset of $\mathcal{A}$.  By a sacrifice function we mean an arbitrary function from set $\mathcal{A}$ to set $[0,+\infty]$. It represents the maximal cost of the sacrifice that each individual agent is ready to bear. 

The language $\Phi$ of our logical system is defined by the grammar
$
\phi:= p \;|\; \neg\phi \;|\; \phi\to\phi \;|\; [C\!:\!X]_s,
$
where $C$ is a coalition, $X$ is a nonempty finite set of formulae, and $s$ is a sacrifice function. We read $[C\!:\!X]_s$ as ``coalition $C$ under sacrifice constraints defined by function $s$ has a dilemma between consciously forcing one of the statements in set $X$ to be true''. For the sake of simplicity, we abbreviate  $[C\!:\!\{\phi_1,\dots,\phi_n\}]_s$ as $[C\!:\!\phi_1,\dots,\phi_n]_s$. We assume that Boolean connectives $\wedge$ and $\vee$ as well as constants truth $\top$ and false $\bot$ are defined as usual. By $\wedge X$ and $\vee X$ we denote the conjunction and the disjunction of all formulae in $X$ respectively. As usual, $\wedge\varnothing$ and $\vee\varnothing$ are defined to be $\top$ and $\bot$, respectively.

\section{Semantics}\label{semantics section}

Throughout this paper, we write $f=_Xg$ if $f(x)=g(x)$ for each $x\in X$.  
We also use shorthand notation captured in the following definition.
\begin{definition}\label{sacrifice notation}
For any game, any complete action profile $\delta$, any state $w$, and any sacrifice function $s$, we write $|\delta|_w\le s$ if $|\delta(a)|_w^a\le s(a)$ for each agent $a\in\mathcal{A}$.
\end{definition}

By a strategy of a coalition $C$ in a given game we mean any function from the set $\Delta^\mathcal{A}$ that assigns an action to each member of the coalition.

Now, we introduce a key definition of this paper. Its part (4) specifies the formal meaning of the multiagent dilemma modality $[C\!:\!X]_s$. Item 4(a) states that any strategy of coalition $C$ forces a specific statement $\phi\in X$ to be true. Item 4(b) states that $X$ is a minimal set with such property.

\begin{definition}\label{sat}
For each game $(W,\Delta,|\cdot|,M,\pi)$, each state $w\in W$, and each formula $\phi\in\Phi$, the satisfaction relation $w\Vdash\phi$ is defined recursively:
\begin{enumerate}
    \item $w\Vdash p$, if $w\in \pi(p)$, where $p$ is a propositional variable,
    \item $w\Vdash\neg\phi$, if $w\nVdash\phi$,
    \item $w\Vdash\phi\to\psi$, if $w\nVdash\phi$ or $w\Vdash\psi$,
    \item $w\Vdash [C\!:\!X]_s$, if 
    \begin{enumerate}
        \item for any strategy $t\in\Delta^C$ of coalition $C$ there is a formula $\phi\in X$ such that for any action profile $\delta\in\Delta^\mathcal{A}$ and any state $u\in W$ if $|\delta|_w\le s$, $t=_C\delta$, and $(w,\delta,u)\in M$, then $u\Vdash\phi$,
        \item for any nonempty subset $Y\subsetneq X$ 
        there is a strategy $t\in\Delta^C$ of coalition $C$ such that for any formula $\phi\in Y$ there is an action profile $\delta\in\Delta^\mathcal{A}$ and a state $u\in W$ where $|\delta|_w\le s$, $t=_C\delta$, $(w,\delta,u)\in M$, and $u\nVdash\phi$.
    \end{enumerate}
\end{enumerate}
\end{definition}

We added the minimality condition 4(b) to the above definition in order to eliminate arbitrary irrelevant alternatives being added to set $X$. We believe that with this condition the definition better reflects our intuition of what a dilemma is. Without item 4(b) the definition would capture the notion of {\em weak dilemma} that we discuss later.

Recall that we allow a game to terminate as a result of agents' actions.
For example, suppose that in a state $w$ an agent $a$ has three actions $d_1$, $d_2$, $d_3$ all of which have a cost of $1$. Let action $d_1$ transition the system into a state in which statement $\phi_1$ is true, action $d_2$ transition the system into a state in which statement $\phi_2$ is true, and action $d_3$ be an action that terminates the game. Then, $w\Vdash [a\!:\!\phi_1,\phi_2]_{a\mapsto 1}$ is true, because each action of agent $a$ predetermines a specific $\phi_i$ to be true in each outcome state. In other words, being able to terminate the system does not provide a way for an agent to ``escape'' the dilemma.

We allow set $X$ in statement $[C\!:\!X]_s$ to be singleton. In such a case, $[C\!:\!X]_s$ is not a dilemma in the common sense of the world, but a ``necessary'' modality. 

\section{Literature Review}

The dilemmas that we study in this paper are usually referred to in the literature as variations of the ``trolley dilemma''. The original trolley dilemma is proposed in~\cite{f67or} as a dilemma faced by an agent who must choose between allowing five people to die and killing one person to prevent the death of those five. 
The distinction between letting one die and killing someone is also emphasised in~\cite{t76m,t84ylj} as well as in~\cite{bb14philosophia}. 
Navar-rete et al. study the same distinction in a virtual reality environment~(\citeyear{nmma12emotion}). 

At the same time, others shift the focus of the trolley dilemma away from the distinction between letting things happen and making things happen. 
Marczyk  and  Marks empirically study whether perceived moral permissibility changed when the person making a decision in the trolley dilemma stands to benefit from or be harmed by one of the outcomes~(\citeyear{mm14ehb}).
Pan and Slater analyse participants' ethical reasoning when they were confronted with the trolley dilemma through an online survey versus through immersive virtual realities~(\citeyear{ps11hci}).
Chen et al. examine the differences in brain activity of Chinese undergraduates who experienced the great Sichuan earthquake when confronted with  trolley dilemmic situations where they must choose to rescue one of two relatives and one of two strangers~(\citeyear{cqlz09bp}).
Indick et al. investigate how the gender of a person affects the decision that she makes in the trolley dilemma-like settings~(\citeyear{ikos00crsp}).
Bleske-Rechek et al. observe that people are less likely to sacrifice the life of one person for the lives of five if the one person is young, a genetic relative, or a current romantic partner~(\citeyear{bnbrb10jsecp}).
In a related work, Kawai, Kubo, and Kubo-Kawai show that most people are inclined to sacrifice an older person over a younger one~(\citeyear{kkk14jpr}).
In this paper, we also consider trolley-like dilemmas in this broader sense. 

Although we are not aware of any works treating dilemma as a modality, there are papers that use existing logical formalism to capture ethical dilemmas. 
Berreby, Bourgne, and Ganascia use simplified event calculus to model dilemmas within answer set programming~(\citeyear{bbg15lpar}).
Horty suggests using nonmonotonic logic for reasoning about moral dilemmas~(\citeyear{h94jpl}).
Bonnemains, Saurel, and Tessier propose formal notations for capturing different ethical norms that can be used in dilemmic settings~(\citeyear{bst18eit}). 

Finally, in this paper we use the cost of a sacrifice as a {\em constraint} on agent's available actions. In a related work, Halpern and Kleiman-Weiner propose to use the cost of a sacrifice as a {\em degree} of blameworthiness~(\citeyear{hk18aaai}).

\section{Ethical Dilemma vs Blameworthiness}

In this section we compare the ethical dilemma modality with blameworthiness modality. We show that the notion of ethical dilemma proposed in this paper cannot be expressed through blameworthiness, as defined through the principle of alternative possibilities: ``a person is morally responsible for what he has done only if he could have done otherwise"~\cite{f69tjop}. In other words, we say that an agent (or a coalition of agents) is responsible for statement $\phi$ if $\phi$ is true and the agent had a strategy to prevent $\phi$. Several formal semantics for blameworthiness as a modality have been proposed in~\cite{nt19aaai,nt20aaai,nt20ai}.

The ethical dilemma modality, just like most other modalities in logic, captures a property of a state. Blameworthiness is not a property of a state, but rather of a transition between states:  statement $\phi$ is true at a {\em current} state $u$, but the agent had a strategy to prevent it in the {\em previous} state $w$. As a result, if the language contains blameworthiness modality, the definition of satisfaction relation $\Vdash$ given in Definition~\ref{sat} should be modified to be a ternary relation $(w,u)\Vdash\phi$ between two states and a formula.

The goal of this section is to show that the dilemma modality cannot be defined through blameworthiness modality. To do this, we first translate the definition of ethical dilemma given in Definition~\ref{sat} into the setting of the two-state satisfaction relation $(w,u)\Vdash\phi$. While doing this, we omit the sacrifice subscript,  assume that the set of agents $\mathcal{A}$ contains a single {\em fixed} agent $a$, and the set of propositional variables contains a single variable $p$. We do this with the intent to show a {\em stronger} result that the dilemma modality is not expressible through blameworthiness modality even in this simple case. In this single-agent setting, we denote coalition strategies and action profiles simply by the action of that fixed agent $a$.  

In this section we consider language $\Phi_0$   described by the grammar
$
\phi:= p \;|\; \neg\phi \;|\; \phi\to\phi \;|\; [a\!:\!X] \;|\; \B_a\phi,
$
where $X$ is any nonempty finite set of formulae in the language $\Phi_0$. We read $\B_a\phi$ as ``agent $a$ is blamable for $\phi$''. The formal semantics for this language is given  below.

\begin{definition}\label{sat-play} 
For each game $(W,\Delta,M,\pi)$, any states $w,u\in W$, and each formula $\phi\in\Phi_0$, the satisfaction relation $(w,u)\Vdash\phi$ is defined recursively:
\begin{enumerate}
    \item $(w,u)\Vdash p$, if $u\in \pi(p)$,
    \item $(w,u)\Vdash\neg\phi$, if $w\nVdash\phi$,
    \item $(w,u)\Vdash\phi\to\psi$, if $w\nVdash\phi$ or $w\Vdash\psi$,
    \item $(w,u)\Vdash [a\!:\!X]$, if 
    \begin{enumerate}
        \item for any action $t\in \Delta$ there is $\phi\in X$ such that for any state $u'\in W$ if $(w,t,u')\in M$, then $(w,u')\Vdash\phi$,
        \item for any nonempty set $Y\subsetneq X$ 
        there is an action $t\in\Delta$ such that for any formula $\phi\in Y$ there is a state $u'\in W$ where $(w,t,u')\in M$, and $(w,u')\nVdash\phi$,
    \end{enumerate}
    \item $(w,u)\Vdash \B_a\phi$, if $(w,u)\Vdash \phi$
    and there is $t\in \Delta$ such that for any state $u'\in W$ if $(w,t,u')\in M$, then $(w,u')\nVdash\phi$.
\end{enumerate}
\end{definition}
Note that items 1 through 4 above are straightforward modifications of corresponding items in Definition~\ref{sat} for a single-agent no-sacrifice language $\Phi_0$. Item 5 captures the principle  of  alternative  possibilities in the same way as in~\cite{nt19aaai}. 

In addition to language $\Phi_0$, we also consider a fragment $\Phi_0^{\mhyphen[~]}$ of $\Phi_0$ that does not use the ethical dilemma modality.


\begin{figure}[ht]
\begin{center}
\scalebox{0.4}{\includegraphics{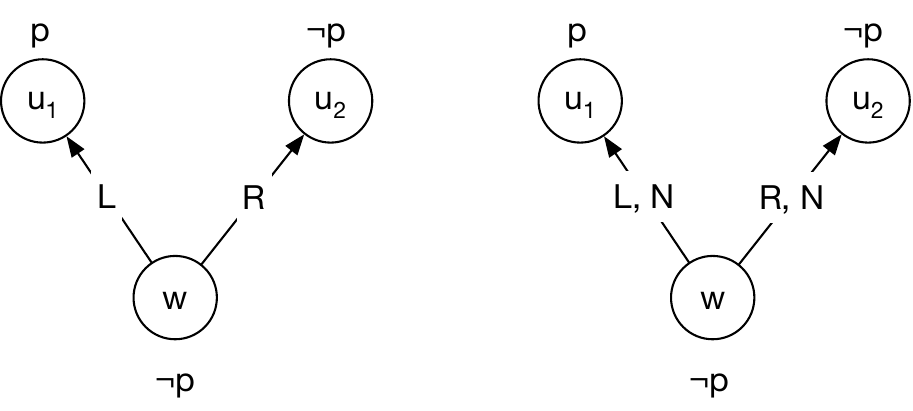}}
\caption{A game.}\label{ethical dilemma figure}
\end{center}
\end{figure}

To show that ethical dilemma modality cannot be defined through the blameworthiness modality, we construct two single-player games that are indistinguishable in language $\Phi_0^{-[~]}$ but are distinguishable in language $\Phi_0$. The two games  are depicted in Figure~\ref{ethical dilemma figure}. These are essentially the same as in our introductory example in Figure~\ref{alice bob dilemma figure}. We will refer to these games as ``left'' and ``right'' games. Both games have three states: $w$, $u_1$, and $u_2$. In both games, propositional variable $p$ is true in state $u_1$ only. In other words, $\pi_l(p)=\{u_1\}=\pi_r(p)$, where $\pi_l$ and $\pi_r$ are valuation functions for the left and the right games respectively. The set of actions in the left game consists of two actions: $L$ and $R$. The right game includes action $N$ in addition to actions $L$ and $R$. The mechanisms $M_l$ and $M_r$ of the left  and the right games respectively are shown in the diagrams using directed edges. For example, the edge from state $w$ to state $u_1$ is labeled with action $L$ on both diagrams. This means that $(w,L,u_1)\in M_l$ and $(w,L,u_1)\in M_r$. We will refer to the satisfaction relations for the left and the right games as $\Vdash_l$ and $\Vdash_r$ respectively.

The next lemma shows that the left and the right games are not distinguishable in language $\Phi_0^{\mhyphen[~]}$.

\begin{lemma}\label{l iff r lemma}
$(w,u)\Vdash_l \phi$ iff $(w,u)\Vdash_r \phi$ for any state $u\in\{u_1,u_2\}$ and formula $\phi\in\Phi_0^{\mhyphen[~]}$. 
\end{lemma}
\begin{proof}
We prove the statement of the lemma by structural induction on formula $\phi$. To prove the statement in case when formula $\phi$ is propositional variable $p$, note that $\pi_l(p)=\{u_1\}=\pi_r(p)$, see Figure~\ref{ethical dilemma figure}. Thus, $(w,u_1)\Vdash_l p$ iff $(w,u_1)\Vdash_l p$ by item 1 of Definition~\ref{sat-play}.

If $\phi$ is a negation or an implication, the desired follows from the induction hypothesis and items 2 and 3 of Definition~\ref{sat-play} in the standard way. Suppose now that formula $\phi$ has the form $\B_a\psi$.

\noindent$(\Rightarrow):$ Let $(w,u)\Vdash_l \B_a\psi$. Thus, by item 5 of Definition~\ref{sat-play}, 
\begin{equation}\label{august02-psi}
    (w,u)\Vdash_l \psi
\end{equation}
and there is an action $t\in\{L,R\}$ such that 
$ 
(w,u')\nVdash_l \psi
$
for any state $u\in \{u_1,u_2\}$ such that $(w,t,u)\in M_l$.
Observe that $\{(w,t,u)\in M_r\;|\; t\in\{L,R\}\}=M_l$, see Figure~\ref{ethical dilemma figure}. Thus, 
$ 
(w,u')\nVdash_l \psi
$
for any state $u\in \{u_1,u_2\}$ such that $(w,t,u)\in M_r$. 
Hence, by the induction hypothesis,
$ 
(w,u')\nVdash_r \psi
$
for any state $u\in \{u_1,u_2\}$ such that $(w,t,u)\in M_r$. 
At the same time, also by the induction hypothesis, statement~(\ref{august02-psi}) implies that  $(w,u)\Vdash_r \psi$. Therefore, $(w,u)\Vdash_r \B_a\psi$ by item 5 of Definition~\ref{sat-play}.

\noindent$(\Leftarrow):$ Assume that $(w,u)\Vdash_r \B_a\psi$. Thus, by item 5 of Definition~\ref{sat-play}, 
\begin{equation}\label{august02-psi2}
    (w,u)\Vdash_r \psi
\end{equation}
and there is an action $t\in\{L,N,R\}$ such that
$ 
(w,u')\nVdash_r \psi
$
for any state $u\in \{u_1,u_2\}$ such that $(w,t,u)\in M_r$. If $t\neq N$, then the prove is similar to the one for the case $(\Rightarrow)$. 

Assume now that $t=N$. In other words,
$ 
(w,u')\nVdash_r \psi
$
for any state $u\in \{u_1,u_2\}$ such that $(w,N,u)\in M_r$. Hence, 
$ 
(w,u')\nVdash_r \psi
$
for any state $u\in \{u_1,u_2\}$,
see Figure~\ref{ethical dilemma figure}.
Thus, by the induction hypothesis,
$ 
(w,u')\nVdash_l \psi
$
for any state $u\in \{u_1,u_2\}$.
In particular,
$ 
(w,u')\nVdash_l \psi
$
for any state $u\in \{u_1,u_2\}$ such that $(w,L,u)\in M_l$.
At the same time, by the induction hypothesis,
statement~(\ref{august02-psi2}) implies that
$(w,u)\Vdash_l \psi$.
Therefore, $(w,u)\Vdash_l \B_a\psi$ by item 5 of Definition~\ref{sat-play}.
\end{proof}

The next two lemmas show that the left and the right models are distinguishable in the language that contains ethical dilemma modality.

\begin{lemma}\label{l lemma}
$(w,u)\Vdash_l[a\!:\!p,\neg p]$ for any state $u\in \{u_1,u_2\}$.
\end{lemma}
\begin{proof}
We verify the two conditions of item 4 of Definition~\ref{sat-play} separately.

\noindent{\em Condition (a):} Consider any $t\in \{L,R\}$. Without loss of generality, let $t=L$. Consider any state $u'\in \{w,u_1,u_2\}$ where $(w,L,u')\in M_l$. To verify the condition, it suffices to show that $(w,u')\Vdash_l p$.

Indeed, assumption $(w,L,u')\in M_l$ implies  $u'=u_1$, see Figure~\ref{sat-play}. Thus, $u'\in \pi_l(p)$, see Figure~\ref{sat-play}. Then, $(w,u')\Vdash_l p$ by item 1 of Definition~\ref{sat-play}.

\noindent{\em Condition (b):} Consider any nonempty set $Y\subseteq\{p,\neg p\}$. Without loss of generality, assume that $Y=\{p\}$. Let $t=R$. To verify the condition, it suffices to prove that there is a state $u'\in \{w,u_1,u_2\}$ such that $(w,t,u')\in M_l$ and $(w,u')\nVdash p$.
Indeed, $u_2\notin \pi_l(p)$, see Figure~\ref{sat-play}. Thus, $u_2\nVdash_l p$ by item 1 of Definition~\ref{sat-play}. At the same time, $(w,R,u_2)\in M_l$, see Figure~\ref{sat-play}.
\end{proof}

\begin{lemma}\label{r lemma}
$(w,u)\nVdash_r[a\!:\!p,\neg p]$ for any state $u\in \{u_1,u_2\}$.
\end{lemma}
\begin{proof}
We will show that condition 4(a) of Definition~\ref{sat-play} does not hold. Indeed, consider strategy $t=N$ and any formula $\phi\in\{p,\neg p\}$. To show that the condition does not hold, it suffices to find state $u'\in \{w,u_1,u_2\}$ such that $(w,N,u')\in M_r$ and $u'\nVdash_r \phi$. Without loss of generality, let $\phi= p$.
Note that $u_2\notin \pi_r(p)$, see Figure~\ref{sat-play}. Thus, $u_2\nVdash_r p$ by item 1 of Definition~\ref{sat-play}. At the same time $(w,N,u_2)\in M_r$, see Figure~\ref{sat-play}.
\end{proof}

The next theorem follows the three previous lemmas.

\begin{theorem}\label{undefinability theorem}
Ethical dilemma modality $[~]$ is not definable in language $\Phi_0^{\mhyphen[~]}$. \qed
\end{theorem}






\section{Ethical Dilemma vs Coalition Power}

Marc Pauly proposed a logic of coalition power that captures properties of modality ``coalition $C$ has a strategy to achieve $\phi$''~\cite{p01illc,p02}. His approach has been widely studied in the literature~\cite{g01tark,vw05ai,b07ijcai,sgvw06aamas,abvs10jal,avw09ai,b14sr,gjt13jaamas,nr18kr}. 
\citeauthor{ahk02} introduced Alternating-Time Temporal Logic (ATL) that combines temporal and coalition modalities~(\citeyear{ahk02}). \citeauthor{gd06tcs} gave a complete axiomatization of ATL~(\citeyear{gd06tcs}). 
\cite{alnr11jlc} introduce resource-bounded coalitional logic (RBCL). 
A logical system with a modality labeled by budget and profit is introduced in~\cite{cn17ijcai}. 

The dilemma modality $[C\!:\!X]_s$, even without the sacrifice subscript $s$, cannot be expressed in the original logic of coalition power. This can be shown using the same two models from Figure~\ref{ethical dilemma figure} that we used to prove Theorem~\ref{undefinability theorem}. However, this modality, {\em without the sacrifice subscript}, can be expressed via {\em socially friendly coalition power} modality introduced in~\cite{ge18aamas}. Its authors proposed several versions of socially friendly modality. The basic one, $[C](\phi;\psi_1,\dots,\psi_n)$ stands for ``coalition $C$ has an action profile that guarantees $\phi$ and enables the complementary coalition $\overline{C}$ to realize any one of $\psi_1,\dots,\psi_k$ by a suitable action profile''. Our modality $[C\!:\!\phi_1,\dots,\phi_n]$ without the sacrifice function is expressible through socially friendly modality as
$
[C](\top;\phi_1,\dots,\phi_n)\wedge
\bigwedge_{D\subsetneq C}\neg [D](\top;\phi_1,\dots,\phi_n). 
$



Unlike ours, the logical system proposed in~\cite{ge18aamas} does not consider cost of actions. Thus, our modality $[C\!:\!X]_s$ with the sacrifice function $s$ is not expressible in their system. They sketch the proof that their axiomatization of socially friendly modality is complete, but, unlike us, do not claim strong completeness. The completeness proofs here and in~\cite{ge18aamas} use different constructions -- see our discussion in~\cite{ny19arxiv}. Additionally, none of the axioms in~\cite{ge18aamas} is similar to our main axiom, the Combination axiom.
Also, recall that the mechanism in  Definition~\ref{game} is nondeterministic. This means that statement $[C\!:\!\phi_1,\dots,\phi_2]$ does {\em not} imply that the complement of coalition $C$ has a strategy to force each of the statements $\phi_1,\dots,\phi_2$.  Goranko and Enqvist's statement $[C](\top,\phi_1,\dots,\phi_n)$ does imply this.

\section{Ethical Dilemma and Imperfect Information}

Recall our introductory example in which an agent is facing a dilemma because she has to make a hard choice between consciously killing Alice and consciously killing Bob. As we discuss there, the agent does not face a dilemma if she can avoid the hard choice by using action $N$ and leaving the outcome up to chance. The other case when the agent does not have to make a hard choice between consciously killing Alice and consciously killing Bob is when she is {\em unaware} of the possible outcomes of her actions. 

\begin{figure}[ht]
\begin{center}
\scalebox{0.4}{\includegraphics{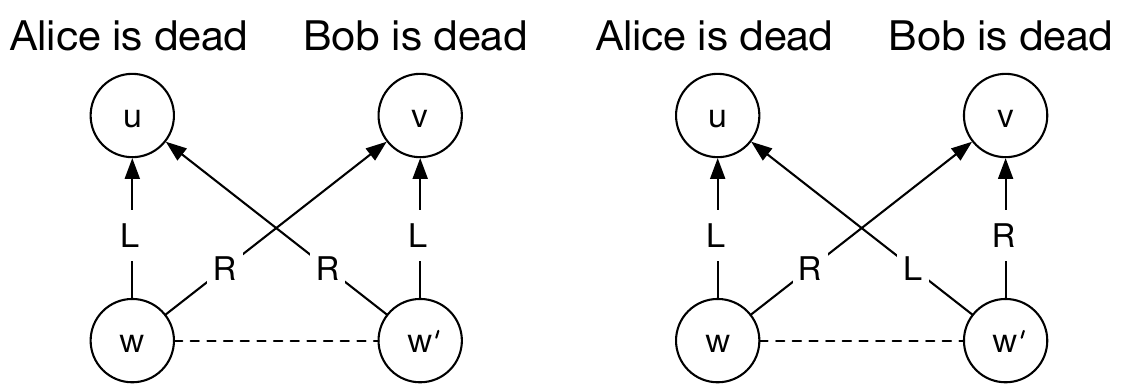}}
\caption{Two settings with imperfect information.}\label{imperfect information dilemma figure}
\end{center}
\end{figure}

Consider, for example, the {\em left} diagram in Figure~\ref{imperfect information dilemma figure}. This diagram depicts an imperfect information game with states $w$ and $w'$ indistinguishable to the agent. In state $w$ the agent has a choice between action $L$ and action $R$. The first of these actions results in Alice's death, the second in Bob's death. However, the agent does not know which action results in whose death because she cannot distinguish state $w$ from state $w'$ where the same actions have the opposite effect. Thus, by choosing one of the two actions in state $w$, the agent does not make a hard choice between consciously killing Alice and consciously killing Bob. We say that she does not face a dilemma in this setting. At the same time, the agent does face a dilemma in the setting depicted in the {\em right} diagram in Figure~\ref{imperfect information dilemma figure} because in both indistinguishable states the actions lead to the same outcome.

To formally define ethical dilemma modalities in imperfect information game settings, one needs to add an indistinguishability equivalence relation $\sim_a$ between states to Definition~\ref{game} of the game. Furthermore, because this definition allows costs of actions to vary from state to state, we need to assume that the cost of the action to an agent $a$ is the same in all $a$-indistinguishable states. In other words, we need to assume that the cost of the action is {\em known} to the agent. 

After the above changes are done to Definition~\ref{game}, one can modify item 4 of Definition~\ref{sat} to capture ethical dilemma in imperfect information setting as as shown below. We write $w\sim_C u$ if $w\sim_a u$ for all agents $a\in C$.

\begin{definition}
For each game $(W,\{\sim_a\}_{a\in\mathcal{A}},\Delta,|\cdot|,M,\pi)$ with imperfect information, each state $w\in W$, and each formula $\phi\in\Phi$, the satisfaction relation $w\Vdash\phi$ is defined recursively:
\begin{enumerate}
    \item[4.] $w\Vdash [C\!:\!X]_s$, if 
    \begin{enumerate}
        \item for any strategy $t\in\Delta^C$ of coalition $C$ there is a formula $\phi\in X$ such that for any action profile $\delta\in\Delta^\mathcal{A}$ and any states $w',u\in W$ if $w\sim_C w'$, $|\delta|_w\le s$, $t=_C\delta$, and $(w',\delta,u)\in M$, then $u\Vdash\phi$,
        \item for any nonempty subset $Y\subsetneq X$ 
        there is a strategy $t\in\Delta^C$ of coalition $C$ such that for any formula $\phi\in Y$ there is an action profile $\delta\in\Delta^\mathcal{A}$ and states $w',u\in W$ where $w\sim_C w'$, $|\delta|_w\le s$, $t=_C\delta$, $(w',\delta,u)\in M$, and $u\nVdash\phi$.
    \end{enumerate}
\end{enumerate}
\end{definition}

Later in this paper we propose a sound and complete logical system for ethical dilemma modality with sacrifice in a perfect information setting. A logical system that describes an interplay between distributed knowledge and blameworthiness in an imperfect information setting is introduced in~\cite{nt20ai}. We leave the development of a similar system for knowledge and dilemmas for the future.

\section{Weak Dilemma}

In the next section we state the axioms of our logical system that capture the properties of modality $[C\!:\!X]_s$. When stating these axioms, it will be convenient to define $\[C\!:\!X\]_s$ as an abbreviation for formula 
$
\bigvee_{\varnothing\neq Z\subseteq X}[C\!:\!Z]_s
$. 
In other words, $\[C\!:\!X\]_s$ means that each action profile of coalition $C$ forces a specific formula in set $X$ to be true, but set $X$ is {\em not necessarily} a minimal such set. We call expression $\[C\!:\!X\]_s$ a {\em weak dilemma}. Alternatively, $\[C\!:\!X\]_s$ could be defined by omitting condition 4(b) from Definition~\ref{sat}.

\section{Axioms}\label{Axioms section}

In this section we list and discuss the axioms and inference rules of our logical system. The first of these axioms uses the notation $X\otimes Y$. For any two sets of formulae $X$ and $Y$, let $X\otimes Y$ be the set of formulae $\{\phi\wedge\psi\;|\; \phi\in X, \psi\in Y\}$.

In addition to propositional tautologies in language $\Phi$, our logical system contains the following axioms:

\begin{enumerate}
    \item Combination:
    $[C\!:\!X]_s\to ([C\!:\!Y]_s\to \[C\!:\!X\otimes Y\]_s)$,
    \item Monotonicity: $[C\!:\!X]_{s'} \to \[D\!:\!X\]_s$, where $C\subseteq D$ and $s\le s'$,
    \item Minimality: $[C\!:\!X]_{s} \to \neg [C\!:\!Y]_s$, where $Y\subsetneq X$,
    \item No Alternatives: $[C\!:\!X]_s\to [D\!:\!X]_s$, where  $|X|=1$.
\end{enumerate}

We write $\vdash\phi$ if formula $\phi\in\Phi$ is derivable in our logical system using the Modus Ponens, the Necessitation, and the Substitution inference rules
$$
\dfrac{\phi,\;\;\;\;\phi\to\psi}{\psi}
\hspace{8mm}
\dfrac{\phi}{[C\!:\!\phi]_s}
\hspace{8mm}
\dfrac{\{\phi\to \tau(\phi)\;|\; \phi\in X\}}{[C\!:\!X]_s\to \[C\!:\!\tau(X)\]_s},
$$
for each function $\tau$ that maps set $\Phi$ into set $\Phi$. If $\vdash\phi$, then we say that formula $\phi$ is a theorem of our system. We write $X\vdash\phi$ if formula $\phi$ is provable from all theorems of our logical system and an additional set of formulae $X$ using the Modus Ponens inference rule only.

The Combination axiom states that if each action profile of coalition $C$ forces a specific formula in set $X$ to be true and a specific formula in set $Y$ to be true, then each action profile of coalition $C$ forces a specific formula in set $X\otimes Y$ to be true. Indeed, if a particular action profile forces $\phi\in X$ to be true and $\psi\in Y$ to be true, then this profile also forces $\phi\wedge\psi$ to be true. 
A hypothetical Combination axiom with the single bracket modality in the conclusion is not sound.



The Monotonicity axiom states that if each action profile of coalition $C$ forces a specific formula in set $X$ to be true under a more relaxed constraint $s'$ on sacrifice, then each action profile of a larger coalition $D$ forces a specific formula in set $X$ to be true under a stronger constraint $s$. A hypothetical Monotonicity axiom with single bracket modality in the conclusion is also not sound. 
The Minimality axiom captures the minimality requirement of item 4(b) in Definition~\ref{sat}.

The No Alternatives axiom deals with the extreme case of a singleton set $X=\{\phi\}$. Note that statement $[C\!:\!\phi]_s$ means that statement $\phi$ is predetermined to be true under any action profile of coalition $C$ as long as actions of {\em all agents } are constrained by $s$. In other words, $\phi$ is true as long as actions of all agents  are constrained by $s$. Since the last statement does not depend on the coalition $C$, we may conclude that validity of statement $[C\!:\!\phi]_s$ does not depend on the choice of coalition $C$. This observation is captured in the No Alternatives axiom.

The Necessitation rule states that if formula $\phi$ is true in all states of all games, then statement $\phi$ is predetermined to be true under any action profile of coalition $C$ and any constraint $s$. Note that in this case the minimality condition 4(b) of Definition~\ref{sat} is vacuously satisfied because singleton set $\{\phi\}$ has no nonempty proper subsets.

The Substitution rule says that if   $[C\!:\!X]_s$ and statement $\phi$ in set $X$ is replaced with a logically weaker statement $\tau(\phi)$, then each action profile of coalition $C$ still forces a specific formula in the set $\tau(X)$ to be true, but $\tau(X)$ is not necessarily the smallest such set. An example of an instance of this rule is
$$
\dfrac{\neg\neg\phi\to\phi,\hspace{5mm} \psi\to(\chi\to\psi)}
{[C\!:\!\neg\neg\phi, \psi]_s\to\[C\!:\!\phi,\chi\to\psi\]_s}.
$$
Note that $X$ and $\tau(X)$ are sets, not lists. Thus,
set $\tau(X)$ might have fewer elements than set $X$:
$$
\dfrac{\phi\to(\phi\vee\psi),\hspace{5mm} \psi\to(\phi\vee\psi)}
{[C\!:\!\phi,\psi]_s\to\[C\!:\!\phi\vee\psi\]_s}.
$$


\begin{theorem}[strong soundness]
If $X\vdash\phi$ and $w$ is a state of a model such that $w\Vdash\chi$ for each formula $\chi\in X$, then $w\Vdash\phi$. \qed
\end{theorem}

The proof of the following completeness theorem can be found in~\cite{ny19arxiv}.

\begin{theorem}[strong completeness]
For any set of formulae $X$ and any formula $\phi$, if $X\nvdash\phi$, then there is a game and a state $w$ of this game such that $w\Vdash\chi$ for each formula $\chi\in X$ and $w\nVdash\phi$. 
\end{theorem}

\section{Conclusion}\label{conclusion section}

The contribution of this paper is three-fold. First, we introduce a formal semantics for ethical dilemmas in a strategic game setting expressed through the modality $[C\!:\!X]_s$. Second, we show that this modality is not definable through the blameworthiness modality. Finally,  we give a complete axiomatization of the  properties of the dilemma modality. 

Our completeness result is the {\em strong} completeness theorem with respect to the proposed semantics. We believe that the standard filtration technique could be used to prove {\em weak} completeness with respect to the class of finite games. This would imply decidability of our logical system, assuming the sacrifice function is rational-valued functions.



\bibliography{sp}

\end{document}